\documentclass{article}
\usepackage[arxiv]{colm2026_conference}

\usepackage{microtype}
\usepackage{amsmath}
\usepackage{amssymb}
\usepackage{mathtools}
\usepackage{amsthm}
\usepackage{booktabs}
\usepackage{graphicx}
\usepackage{wrapfig}
\usepackage{subcaption}
\usepackage{xcolor}
\usepackage[most,skins,theorems]{tcolorbox}
\usepackage{tikz}
\usepackage{enumitem}
\usepackage{url}
\usepackage{hyperref}
\usepackage{cleveref}

\usetikzlibrary{arrows.meta,positioning,fit}

\definecolor{darkblue}{rgb}{0, 0, 0.5}
\definecolor{promptbg}{RGB}{232,242,255}
\definecolor{promptframe}{RGB}{92,140,214}
\definecolor{userpromptbg}{RGB}{232,242,255}
\definecolor{userpromptframe}{RGB}{92,140,214}
\definecolor{developerpromptbg}{RGB}{244,236,255}
\definecolor{developerpromptframe}{RGB}{124,82,171}
\definecolor{evalbg}{RGB}{241,252,241}
\definecolor{evalframe}{RGB}{77,162,94}
\definecolor{solutionbg}{RGB}{255,246,232}
\definecolor{solutionframe}{RGB}{214,139,42}

\hypersetup{
  colorlinks=true, citecolor=darkblue, linkcolor=darkblue, urlcolor=darkblue,
  pdftitle={ProofCouncil: An LLM Agent for Solving Open Mathematical Problems},
  pdfauthor={Johannes Schmitt, Tim Gehrunger, Jasper Dekoninck, Gergely Bérczi, Uri Kreitner, Liam Price, David Holmes}
}
\setcitestyle{authoryear,round,citesep={;},aysep={,},yysep={;}}

\newtcolorbox{promptbox}[1]{
  enhanced,
  breakable,
  arc=2.2mm,
  boxrule=0.8pt,
  left=1.3mm,right=1.3mm,top=1.0mm,bottom=1.1mm,
  colback=promptbg,
  colframe=promptframe,
  colbacktitle=promptframe!34,
  coltitle=black,
  fonttitle=\bfseries,
  title={#1}
}

\newtcolorbox{userpromptbox}[1]{
  enhanced,
  breakable,
  arc=2.2mm,
  boxrule=0.8pt,
  left=1.3mm,right=1.3mm,top=1.0mm,bottom=1.1mm,
  colback=userpromptbg,
  colframe=userpromptframe,
  colbacktitle=userpromptframe!34,
  coltitle=black,
  fonttitle=\bfseries,
  title={#1}
}

\newtcolorbox{developerpromptbox}[1]{
  enhanced,
  breakable,
  arc=2.2mm,
  boxrule=0.8pt,
  left=1.3mm,right=1.3mm,top=1.0mm,bottom=1.1mm,
  colback=developerpromptbg,
  colframe=developerpromptframe,
  colbacktitle=developerpromptframe!30,
  coltitle=black,
  fonttitle=\bfseries,
  title={#1}
}

\newtcolorbox{evalbox}[1]{
  enhanced,
  breakable,
  arc=2.2mm,
  boxrule=0.8pt,
  left=1.3mm,right=1.3mm,top=1.0mm,bottom=1.1mm,
  colback=evalbg,
  colframe=evalframe,
  colbacktitle=evalframe!30,
  coltitle=black,
  fonttitle=\bfseries,
  title={#1}
}

\newtcolorbox{solutionbox}[1]{
  enhanced,
  breakable,
  arc=2.2mm,
  boxrule=0.8pt,
  left=1.3mm,right=1.3mm,top=1.0mm,bottom=1.1mm,
  colback=solutionbg,
  colframe=solutionframe,
  colbacktitle=solutionframe!34,
  coltitle=black,
  fonttitle=\bfseries,
  title={#1}
}

\crefname{section}{Sec.}{Secs.}
\crefname{subsection}{Sec.}{Secs.}
\crefname{subsubsection}{Sec.}{Secs.}
\crefname{paragraph}{Sec.}{Secs.}
\crefname{appendix}{App.}{Apps.}
\crefname{equation}{Eq.}{Eqs.}
\crefname{figure}{Fig.}{Figs.}
\crefname{table}{Tab.}{Tabs.}

\newtheorem{theorem}{Theorem}[section]
\newtheorem{proposition}[theorem]{Proposition}
\newtheorem{lemma}[theorem]{Lemma}

\crefname{theorem}{Thm.}{Thms.}
\Crefname{theorem}{Theorem}{Theorems}
\crefname{proposition}{Prop.}{Props.}
\Crefname{proposition}{Proposition}{Propositions}
\crefname{lemma}{Lemma}{Lemmas}
\Crefname{lemma}{Lemma}{Lemmas}
\crefname{remark}{Rem.}{Rems.}
\Crefname{remark}{Remark}{Remarks}

\providecommand{\Z}{\mathbb{Z}}
\providecommand{\one}{\mathbf{1}}
\providecommand{\gcdop}{\operatorname{gcd}}

\title{ProofCouncil: An LLM Agent for Solving Open Mathematical Problems}

\author{
\makebox[\textwidth][c]{%
\begin{tabular}[t]{c}
Johannes Schmitt$^1$, Tim Gehrunger$^1$, Jasper Dekoninck$^1$, \\
Gergely B\'erczi$^2$, Uri Kreitner$^1$, Liam Price$^3$, David Holmes$^4$ \\
{\normalfont $^1$ETH Zurich, $^2$Aarhus University, $^3$Independent Researcher, $^4$Leiden University}\\
{\normalfont Correspondence: \texttt{johannes.schmitt@math.ethz.ch}}
\end{tabular}}
}

\begin{document}

\maketitle
\vspace{-7mm}
\begin{center}
\raisebox{-0.16em}{\includegraphics[height=1em]{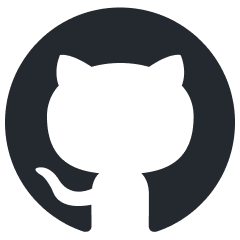}} \url{https://github.com/eth-sri/proof-council}\\
\end{center}
\vspace{2mm}

\begin{abstract}
Large language models (LLMs) have shown increasing promise in solving open problems in mathematics. However, their performance can be further improved through agentic workflows tailored to real-world mathematical practice. To this end, we introduce ProofCouncil, a mathematical agent that is designed to tackle open problems using an author-critic architecture. ProofCouncil served as a submission to the second batch of FirstProof, a challenge consisting of 10 real-world mathematical problems that agents must solve autonomously. Its submissions for 6 of the 10 problems were judged by the referees to be correct up to at most minor revisions, showing the best performance among participating teams. We also evaluate ProofCouncil on 30 open problems collected from mathematical researchers. Among the 21 solutions that received human feedback, 5 were judged completely correct, 2 more were judged promising pending final verification, and a further 8 contained useful partial progress. In this short paper, we describe the development of ProofCouncil and the agent-building library used to create it, which we release as open source to the community.
\end{abstract}

\section{Introduction}
\label{sec:introduction}

Large language models (LLMs) have recently demonstrated remarkable capabilities in mathematics by solving open problems of varying interest \citep{alon2026remarks,schmitt2025extremaldescendantintegralsmoduli,aletheia}. This progress has been accompanied by the development of advanced agents designed for mathematical problem-solving \citep{aletheia,zhao2026rma,zheng2026ai,ju2026automated}, further highlighting the potential of LLMs in this domain. In this context, the second iteration of the FirstProof challenge \citep{firstproof,abouzaid2026proofsecondbatch} was launched to evaluate AI systems on a fixed set of 10 problems with no publicly available solution. Participants were required to submit an agent that autonomously attempted all 10 problems within a 24-hour window, using only publicly available models and tools.

In this work, we present ProofCouncil, a mathematical agent submitted to the FirstProof challenge. As shown in \cref{fig:overview}, ProofCouncil follows an author-critic architecture: an author agent iteratively writes and revises a proof, while a critic agent evaluates each version and provides feedback. At every round, the author may also request targeted assistance from two auxiliary sources: a council of additional LLMs outside the main author-critic loop, and a compute agent for computer algebra system (CAS) computations. Their responses are provided to the author in the next iteration.
To maintain consistency and allow the critic to track progress, the critic generally retains its conversation history across rounds, but this history is reset every $k$ rounds to provide a more independent review. If a freshly initialized critic accepts the proof, the proof is returned to the user.

We open-source both ProofCouncil and the underlying agent-building library. The library represents agentic systems as conditional directed acyclic graphs (DAGs), enabling flexible structures of calls between different agents and models. It also provides a simple interface for implementing new agents and workflows, enabling future researchers to build on our work and create their own agentic systems.

In the FirstProof challenge, ProofCouncil solved 6 of the 10 problems up to at most minor revisions and showed the best performance among participating teams. As a secondary evaluation, we also reached out to mathematical researchers and collected 30 open problems across various fields of mathematics. Solution attempts were sent to the question authors for feedback. Among 21 received responses, 5 were judged to be correct solutions, 2 were considered promising pending final verification, and 8 more contained useful partial progress towards a solution. Notably, no expert reported that an output claiming to solve the stated problem was mathematically incorrect. However, two outputs were judged to have solved easier versions of the submitted problem, underscoring problem interpretation as a key failure mode. We do not release the solutions to these researcher-provided problems, as they are the intellectual property of the researchers who shared them with us and may be used by them at their discretion.

Our key contributions are as follows:
\begin{itemize}
\item We introduce ProofCouncil, a mathematical agent designed to solve complex problems through an author-critic architecture (\cref{sec:architecture}).
\item We develop and release an open-source agent-building library that enables flexible and efficient construction of agentic systems as conditional DAGs (\cref{sec:library}).
\item We demonstrate the effectiveness of ProofCouncil by solving 6 out of 10 problems (up to at most minor revisions) in the FirstProof challenge and producing correct solutions to 5 open problems collected from mathematical researchers (\cref{sec:evaluation}).
\end{itemize}

\begin{figure}[t]
\centering
\providecommand{\lexeme}[1]{\texttt{#1}}
\definecolor{accentblue}{RGB}{88,139,202}
\definecolor{lightgreen}{RGB}{220,245,230}
\definecolor{lightred}{RGB}{250,225,225}
\definecolor{darkgreen}{RGB}{40,167,69}
\definecolor{darkred}{RGB}{220,53,69}
\definecolor{textgray}{RGB}{120,120,120}
\definecolor{lightgray}{RGB}{240,240,240}
\definecolor{gray2}{HTML}{FCFCFC}
\definecolor{overviewblue}{HTML}{347bc6}
\definecolor{purple}{RGB}{155, 89, 182}
\usetikzlibrary{arrows.meta,positioning,fit,shadows.blur}

\noindent\resizebox{0.98\linewidth}{!}{%
\begin{tikzpicture}[
    node distance=7mm and 8mm,
    base/.style={
        draw=gray,
        fill=lightgray,
        rounded corners=3pt,
        font=\sffamily,
        blur shadow={shadow xshift=0.5pt, shadow yshift=-0.5pt, shadow opacity=20}
    },
    stagebox/.style={
        base,
        minimum width=2.25cm,
        minimum height=0.95cm,
        align=center,
        text width=2.25cm,
        inner sep=3pt,
        font=\sffamily\small
    },
    agentbox/.style={
        stagebox,
        fill=accentblue!20,
        draw=accentblue,
        minimum width=2.45cm
    },
    authorbox/.style={
        stagebox,
        fill=purple!13,
        draw=purple
    },
    filebox/.style={
        stagebox,
        fill=lightgray!20,
        draw=gray,
        text width=2.25cm,
        minimum width=2.45cm
    },
    auxbox/.style={
        stagebox,
        fill=gray2,
        draw=accentblue!65,
        text width=2.0cm,
        minimum height=0.82cm,
        font=\sffamily\scriptsize
    },
    packetbox/.style={
        stagebox,
        fill=accentblue!12,
        draw=accentblue!75,
        text width=2.1cm
    },
    gatebox/.style={
        stagebox,
        fill=lightgray!20,
        draw=gray,
        text width=2.15cm,
        minimum width=2.35cm
    },
    resultbox/.style={
        base,
        minimum height=0.58cm,
        minimum width=1.55cm,
        align=center,
        font=\sffamily\bfseries\small
    },
    accept/.style={resultbox, fill=lightgreen, draw=darkgreen, text=darkgreen},
    reject/.style={resultbox, fill=lightred, draw=darkred, text=darkred},
    label/.style={
        font=\sffamily\tiny,
        text=textgray,
        align=center,
        inner sep=1pt
    },
    arrow/.style={
        ->,
        thick,
        line width=0.4mm,
        color=overviewblue!80,
        >={Triangle[scale=0.6]}
    },
    greenarrow/.style={
        arrow,
        color=darkgreen!85
    },
    optional/.style={
        arrow,
        dashed,
        color=overviewblue!58
    },
    feedback/.style={
        arrow,
        color=darkred!75
    }
]

\node[stagebox, fill=gray2, draw=gray] (problem) at (0,0) {\textbf{Problem}\\[-1pt]{\scriptsize statement}};

\node[authorbox, right=0.5cm of problem] (author) {\textbf{Author}\\[-1pt]{\scriptsize edits proof}};

\node[agentbox, right=3.4cm of author, yshift=0.65cm] (critic) {\textbf{Stateful Critic}\\[-1pt]{\scriptsize Same conversation}};
\node[agentbox, right=3.4cm of author, yshift=-0.65cm] (fresh) {\textbf{Fresh Critic}\\[-1pt]{\scriptsize New conversation}};
\node[auxbox, right=0.72cm of author, yshift=1.15cm] (council) {\textbf{LLM Council}\\[-1pt]Other LLMs};
\node[auxbox, right=0.72cm of author, yshift=-1.15cm] (compute) {\textbf{Compute Node}\\[-1pt]CAS worker};

\node[accept, right=0.65cm of fresh] (ship) {Return};



\draw[arrow] (problem.east) -- (author.west);
\coordinate (criticsplit) at ([xshift=0.58cm]author.east);
\draw[arrow] (author.east) -- (criticsplit) |- node[label, below, pos=0.74] {other rounds} (critic.west);
\draw[arrow] (criticsplit) |- node[label, above, pos=0.74] {Every $k$ rounds} (fresh.west);
\draw[optional] (author.east) -- ++(0.29,0) |- (council.west);
\draw[optional] (author.east) -- ++(0.29,0) |-  (compute.west);
\draw[optional] (council.north) -- ++(0,0.1) -| ([xshift=0.25cm]author.north);
\draw[optional] (compute.south) -- ++(0,-0.1) -| ([xshift=0.25cm]author.south);

\draw[greenarrow] (critic.south) -- node[label, right, pos=0.5] {fresh audit} (fresh.north);
\draw[feedback] (critic.north) -- node[label, right, pos=0.5] {rejects proof} ++(0,0.65) -| (author.north);

\draw[greenarrow] (fresh.east) -- (ship.west);
\draw[feedback] (fresh.south) -- node[label, right, pos=0.5] {rejects proof} ++(0,-0.65) -| (author.south);



\end{tikzpicture}%
}
\caption{Overview of ProofCouncil. The author agent iteratively edits the proof in response to feedback from a critic. Every $k$ rounds, the stateful critic is reset. The author may optionally request help from other LLMs or a compute node.}
\vspace{-1mm}
\label{fig:overview}
\end{figure}
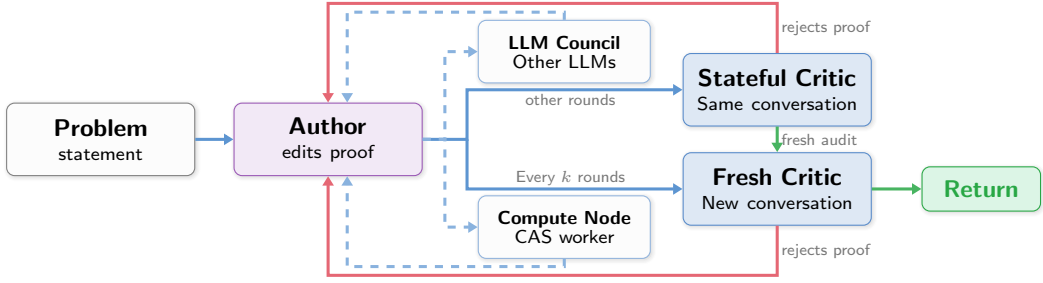

\section{ProofCouncil}
\label{sec:system}

In this section, we first provide an overview of the system architecture of ProofCouncil and then discuss the library we developed to build ProofCouncil as conditional DAGs of agent calls. Prompts are provided in \cref{app:prompts}.

\subsection{ProofCouncil Workflow}
\label{sec:architecture}

The overall workflow of ProofCouncil is shown in \cref{fig:overview}. The system proceeds iteratively until one of four stopping conditions is met: a solution is accepted, the maximum number of rounds is reached, the total cost budget is exhausted, or the timeout is reached. We describe each component in turn.

\vspace{-1mm}
\paragraph{Author}
The author is responsible for producing and revising the proof. It uses GPT-5.5-Pro (xhigh)\footnote{Parenthetical labels such as (xhigh) or (high) denote the reasoning-effort setting used for the model calls.} with access to built-in code execution and web search tools. Through code execution, the author maintains and edits three files:
\begin{itemize}
\item \texttt{answer.tex}: the main \LaTeX{} file containing the proof.
\item \texttt{research\_notes.tex}: a scratchpad for ideas, failed attempts, and partial results that are not part of the main proof but may help the author or critic track progress.
\item \texttt{references.bib}: a BibTeX file containing the references cited in the proof.
\end{itemize}

At each iteration, the author receives the latest critic feedback, any auxiliary information produced by the council or compute node, and the current state of all files. It then updates the proof and, when useful, its research notes and bibliography.

The author can also request auxiliary help by emitting structured XML tags in its output. A \texttt{<council>} tag contains a specific question to the LLM council, while a \texttt{<compute\_agent>} tag contains a question for the compute node. The resulting responses are passed back to the author in the next iteration. Finally, the author must emit a \texttt{<ready>} tag indicating whether it believes it has found a complete and correct solution.

\vspace{-1mm}
\paragraph{Critic}
The critic is responsible for assessing the correctness of the current proof. It receives the files written by the author and provides feedback on gaps, errors, or missing justifications. Like the author, it uses GPT-5.5-Pro (xhigh) with access to the same tools.

By default, the critic is stateful. Each new version of the proof is appended to the critic's conversation history, allowing the critic to provide consistent feedback. To reduce path dependence and obtain more independent feedback, we reset the critic every $k$ rounds. In our experiments, we use $k=3$.

The critic also emits an \texttt{<answer\_ready>} tag indicating whether it considers the proof complete and correct. If the stateful critic accepts the proof, a fresh critic is called for an additional audit. The proof is returned as a solution only if the fresh critic also accepts it, and the author has marked the answer as ready.

\vspace{-1mm}
\paragraph{LLM council}
The LLM council is an optional auxiliary component that provides additional feedback outside the main author-critic loop. When called, each council member receives all files written by the author, together with the specific question posed by the author. The council currently consists of Gemini 3.1 Pro (high), Claude Opus 4.7 (max), and GPT-5.5-Pro (xhigh). The models answer independently, without seeing one another's responses. Their answers are then given to the author as additional context in the next iteration.

\vspace{-1mm}
\paragraph{Compute node}
The compute node is an optional auxiliary component for code execution and computer algebra system (CAS) computations. It is implemented as a Codex agent using GPT-5.5 (xhigh), with access to specialized mathematical software including SageMath \citep{sagemath}, GAP \citep{GAP4}, Singular \citep{DGPS}, and Pari \citep{PARI2}. The compute node receives all files written by the author, together with the author's specific question and returns a response that may include code, CAS output, or explanatory analysis. This response is provided to the author in the next iteration. In addition, the edited code repository is added to the file system available to the author, allowing it to inspect or reuse the compute node's outputs in subsequent reasoning.

\subsection{Agents as DAGs}
\label{sec:library}

To create ProofCouncil, we developed an open-source library for constructing agentic systems as conditional directed acyclic graphs (DAGs). We briefly describe its main features here.

\vspace{-1mm}
\paragraph{DAG representation}
In our library, an agent is represented as a DAG whose nodes can be LLM calls, custom Python blocks, CLI-based compute agents, deterministic gates (such as compilation checks), or control-flow blocks, with edges encoding the dependencies between them. Nodes may themselves be agents, enabling hierarchical compositions of models and workflows. An LLM node, for instance, is defined by parameters such as a prompt template, a model name, and a set of outputs. The prompt is rendered at runtime using the outputs of upstream nodes, and the outputs produced by a node can in turn serve as inputs to downstream nodes. This abstraction supports complex information flow while keeping the structure of the system explicit.

\vspace{-1mm}
\paragraph{Conditional execution}
Plain DAGs are not expressive enough for many agentic workflows because the set of calls to execute may depend on intermediate outputs. For instance, in ProofCouncil, the LLM council is called only when the author emits a corresponding tag. To support this behavior, edges can be equipped with conditions that determine whether a downstream node should execute based on the outputs of its parents.

Although the workflow in \cref{fig:overview} contains a feedback loop from the critic to the author, it can be represented as a conditional DAG by unrolling the loop for a fixed number of rounds. For readability, the library supports bounded repeat blocks that represent loops, producing a structure that is easier to read and edit while still executing as a finite conditional DAG.

\paragraph{Agent execution}
To execute an agent, the library computes a topological order of the nodes in the DAG. A node is run as soon as all of its dependencies have completed and its edge conditions are satisfied. This enables independent nodes to execute in parallel, improving efficiency in workflows with multiple auxiliary calls or independent branches.

\paragraph{User interface}
We also built a custom user interface for editing and running agents, inspired by Blender nodes \citep{blender}. The interface allows users to visually construct and modify the DAG structure, edit prompt templates, add or remove outputs, and monitor the execution of each node. Outputs are specified by requiring models to emit structured YAML tags. Custom components can be added as nodes, and control-flow constructs such as conditionals and repeat blocks can be created visually. Screenshots are provided in \cref{app:library}.

\paragraph{Human-in-the-loop workflows}
The library also supports human-in-the-loop workflows: a human can act as a node in the DAG, interchangeable with any model node. When such a node is reached, the run pauses, the task is presented in the user interface, and execution resumes with the submitted answer. This enables a mathematician, among others, to review intermediate proofs or answer the author's questions mid-run.

\section{Evaluation}
\label{sec:evaluation}

We evaluate ProofCouncil on two sets of problems: a set of open problems collected from mathematical researchers, primarily at ETH Zurich, and the 10 problems from FirstProof. The former set was collected to test the system on a broader and more diverse collection of research problems. We also used feedback from these runs to improve ProofCouncil.

\begin{table}[t]
\centering
\vspace{-5mm}
\caption{Summary of human feedback on the open-question evaluation.}
\vspace{-2mm}
\label{tab:open-question-feedback}
\begin{tabular}{@{}lc@{}}
\toprule
Category & Count \\
\midrule
Problems evaluated & 30 \\
Researcher reviews received & 21 \\
\midrule
Complete solutions & 5 \\
Possibly complete solutions & 2 \\
Meaningful partial progress & 8 \\
No errors, but no progress & 4 \\
Misinterpreted problems & 2 \\
\bottomrule
\end{tabular}
\end{table}

\subsection{Results on Open Questions}

\begin{figure}[t]
\centering
\includegraphics[width=\textwidth]{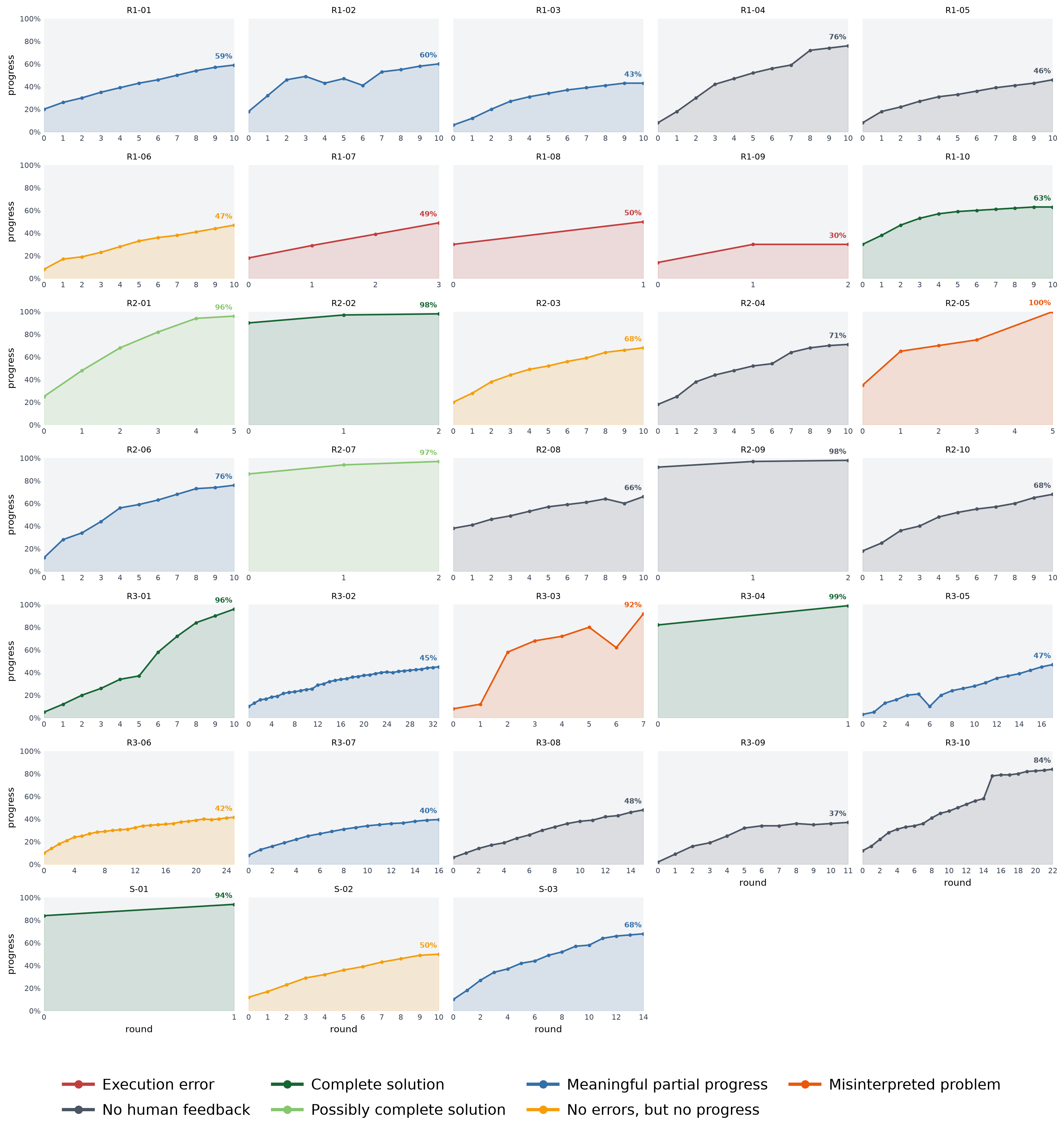}
\caption{Progress panel for all 33 attempted problems: rows R1--R3 show the three 10-problem runs, row S the three single-problem pretest runs. Curves show GPT-5.5-Pro postscreen progress estimates. Colors indicate final status, per human feedback where received; the feedback categories match \cref{tab:open-question-feedback}.}
\label{fig:progress-panel-all}
\end{figure}

To prepare our submission for the FirstProof challenge, we ran several preliminary evaluations: three single-problem pretest runs on researcher-submitted problems, followed by three larger runs of 10 problems each, the last of which also included the three Erd\H{o}s problems. In this section, we report the results of these runs and summarize the feedback we received from human mathematicians on the solutions produced by ProofCouncil.

\paragraph{Problem collection}
To collect problems, we contacted all members of the mathematics department at ETH Zurich and asked them to share open problems they were interested in seeing attempted by an AI system. Further problems were sent by researchers at a few other institutions via private communication. Each problem was screened using GPT-5.5-Pro to check whether it was self-contained and accurately stated. When necessary, edits were made in collaboration with the submitting researcher. The screening prompt is provided in \cref{app:prompts}.
Across all runs, we evaluated ProofCouncil on 33 distinct problems: 30 problems collected from researchers and 3 public Erdős problems.  In \cref{sec:erdos539-case-study}, we give a self-contained account of one Erdős-problem result from these runs.

\paragraph{Problem execution}
Each run used the same codebase, with minor prompt changes and bug fixes introduced between runs based on earlier results. Importantly, three problems in the first 10-problem run encountered execution errors. In the final 10-problem run, we increased the maximum number of iterations from 10 to 200, effectively allowing the system to run until the 24-hour time limit was reached.

\paragraph{Main results}

We requested feedback from problem authors on all produced outputs for the non-Erdős problems (shown in \cref{tab:open-question-feedback}). Among the 21 responses received so far, 5 outputs were judged to be complete solutions,\footnote{One of these problems was submitted by an author of this paper. To ensure independence, the solution was given to three external experts in the area, who judged the results new and the argument likely correct. The proof has since been substantially verified by the submitting author together with these experts, in the course of preparing a follow-up publication.} and 2 were considered promising but not yet fully verified by the researchers. Most of the remaining outputs were not complete, but still useful: 8 contained partial progress, while 4 contained no apparent errors but offered little substantive progress. The remaining 2 responses had misinterpreted the problem, producing correct solutions to easier versions of the stated problems. Importantly, among the 21 received reviews, no reviewer reported that an output claiming to solve the stated problem was mathematically false. Instead, the two misinterpretation cases highlight problem interpretation as a significant failure mode. The recorded cost was approximately \$213 per attempted problem. 

\paragraph{Qualitative human feedback}
Several researchers expressed surprise at the quality of the outputs, noting that solutions were useful, constrained the search space for future work, or took the right approach across all subquestions. In one notable case, ProofCouncil found a complete solution using an unexpected approach and went substantially further than the researcher had anticipated for the project.
At the same time, researchers also identified important limitations. Some outputs relied on an unorthodox interpretation that made the problem easier or nearly trivial. Others mainly explained why the problem was difficult without making substantial progress. In some cases, the mathematical content was promising, but the write-up did not read like a polished research paper.

\paragraph{Postscreen audit}
To further analyze the runs, we asked GPT-5.5-Pro to review the intermediate logs and assign a progress score to each problem at each round. The resulting progress panel is shown in \cref{fig:progress-panel-all}. The audit suggests that some problems were only marked as solved after several rounds of feedback, indicating that the iterative author-critic process was important in reaching the final solution. However, this analysis was solely based on the model's interpretation of the logs, and the results should be interpreted with caution.

The audit also revealed recurring failure modes. In particular, ProofCouncil often became stuck in a local minimum, repeatedly attempting to patch a specific issue in the proof without making meaningful progress toward a complete solution. One possible way to address this would be to run several author-critic threads in parallel, allowing the system to explore more diverse proof strategies, but this would substantially increase cost. Finally, we note that both auxiliary components were used extensively: the compute node and the LLM council each appeared in more than 50\% of rounds.

\subsection{Results on FirstProof}

\subsubsection{Overall results}

In the FirstProof challenge, expert referees appointed by the organizers graded the submissions to the 10 problems.\footnote{In the official FirstProof report \citep{abouzaid2026proofsecondbatch}, our submitted harness appears as System~A (IMProofBench ProofCouncil).} We refer to the problems as P1, \dots, P10. Based on the referee recommendations, submissions for 6 of the 10 problems (P1, P2, P3, P5, P7, and P9) were judged correct up to at most minor revisions. One further submission, P10, contained partial progress toward a solution but required major revisions. Of the remaining three problems (P4, P6, and P8), the submissions for P4 and P8 were rejected, while ProofCouncil produced no submission for P6 because of repeated timeout errors from the OpenAI API from which the system did not recover. We therefore exclude P6 from the remaining analysis, since it is not informative about ProofCouncil's performance.

\paragraph{Cost.}
The run logged approximately \$350 per problem in model-call costs.\footnote{The total logged cost was \$3{,}186 (\cref{fig:firstproof-component-costs}); we divide by the nine analyzed problems, since P6 produced no submission.} In \cref{fig:firstproof-component-costs}, we show the distribution of costs across the main ProofCouncil components. By comparison, the single-query GPT-5.5-Pro baseline, one of the four systems evaluated in the challenge, cost only about \$12 per problem, but solved 4 of the 9 analyzed problems rather than 6.  Thus, ProofCouncil improved the solve rate, but at substantially higher total cost. Since we did not optimize ProofCouncil for cost, it is likely that a more efficient configuration could achieve similar results at lower cost, which we leave for future work.

\subsubsection{Component utility}
We next discuss several notable observations about the non-author components from the run.

\paragraph{Critic.}
On almost all problems, the critic helped prevent incorrect solutions from being accepted. Verification by a fresh critic was especially important on P10, where the proposed solution contained partial progress but was judged incomplete by the referees: the stateful critic accepted intermediate versions of the solution at seven different iterations, whereas the fresh critic rejected all of them. This suggests that independent re-evaluation can be important when the critic has been exposed to other history.

There were two notable failure cases. First, on P8, a critic accepted an incorrect solution that was later rejected by the human referees because it relied on an unsupported technical assertion not proved by the cited references. Second, on P3, the critic rejected the solution at every stage, even though the referees judged the output correct up to minor revisions.

We further investigated the P8 false positive using a staged GPT-5.5-Pro diagnostic. Given only the ProofCouncil submission, the model partially identified the critical part of the proof as the weak point, but judged the solution as requiring only minor revisions. After being shown the human solution, it revised its assessment to major revisions and recognized that the missing component was substantial. After seeing the referee reports, it agreed with high confidence that the central assertion had not been proved and was not supported by the cited references.

\paragraph{Council.}
GPT-5.5-Pro was the strongest and most expensive council member, and contributed visibly useful ideas on P3, P5, and P10. Claude Opus also made a concrete contribution on P3 by suggesting a correct direction, and supplied useful comments on P5. Gemini was substantially cheaper and less decisive, but still provided useful route validation on P5 and pointed toward relevant analogies on P3.


\paragraph{Compute worker.}
The compute worker was used most extensively on P3, P4, P5, P8, and P10. Its main roles were literature search and direct verification of claims arising in evolving drafts. In several cases, it found counterexamples or failures of proposed intermediate assertions. This suggests that the compute worker was most useful when the author loop generated concrete claims that could be checked independently, rather than when broad strategic progress was needed.

\begin{figure}[t]
\centering
\begin{tikzpicture}[font=\small]
\definecolor{costauthor}{RGB}{68,119,170}
\definecolor{costcritic}{RGB}{221,132,82}
\definecolor{costgpt}{RGB}{85,168,104}
\definecolor{costclaude}{RGB}{196,78,82}
\definecolor{costgemini}{RGB}{129,114,179}
\definecolor{costcompute}{RGB}{128,133,133}
\def\costradius{1.65}
\filldraw[fill=costauthor,draw=white,line width=0.6pt] (0,0) -- (90:\costradius) arc[start angle=90,end angle=265.710,radius=\costradius] -- cycle;
\filldraw[fill=costcritic,draw=white,line width=0.6pt] (0,0) -- (265.710:\costradius) arc[start angle=265.710,end angle=379.638,radius=\costradius] -- cycle;
\filldraw[fill=costgpt,draw=white,line width=0.6pt] (0,0) -- (379.638:\costradius) arc[start angle=379.638,end angle=420.453,radius=\costradius] -- cycle;
\filldraw[fill=costclaude,draw=white,line width=0.6pt] (0,0) -- (420.453:\costradius) arc[start angle=420.453,end angle=433.502,radius=\costradius] -- cycle;
\filldraw[fill=costgemini,draw=white,line width=0.6pt] (0,0) -- (433.502:\costradius) arc[start angle=433.502,end angle=435.500,radius=\costradius] -- cycle;
\filldraw[fill=costcompute,draw=white,line width=0.6pt] (0,0) -- (435.500:\costradius) arc[start angle=435.500,end angle=450.000,radius=\costradius] -- cycle;
\node[anchor=west] at (2.25,1.45) {\textbf{Component spend (Total: \$3,186)}};
\fill[costauthor] (2.25,0.94) rectangle ++(0.18,0.18);
\node[anchor=west] at (2.52,1.03) {Author: \$1,555 (48.8\%)};
\fill[costcritic] (2.25,0.56) rectangle ++(0.18,0.18);
\node[anchor=west] at (2.52,0.65) {Critic, fresh and stateful: \$1,008 (31.6\%)};
\fill[costgpt] (2.25,0.18) rectangle ++(0.18,0.18);
\node[anchor=west] at (2.52,0.27) {Council: GPT-5.5-Pro: \$361 (11.3\%)};
\fill[costclaude] (2.25,-0.20) rectangle ++(0.18,0.18);
\node[anchor=west] at (2.52,-0.11) {Council: Claude Opus 4.7: \$115 (3.6\%)};
\fill[costgemini] (2.25,-0.58) rectangle ++(0.18,0.18);
\node[anchor=west] at (2.52,-0.49) {Council: Gemini 3.1 Pro Preview: \$18 (0.6\%)};
\fill[costcompute] (2.25,-0.96) rectangle ++(0.18,0.18);
\node[anchor=west] at (2.52,-0.87) {Compute worker: \$128 (4.0\%)};
\end{tikzpicture}
\caption{ProofCouncil model spend by component in the official FirstProof run.}
\label{fig:firstproof-component-costs}
\end{figure}
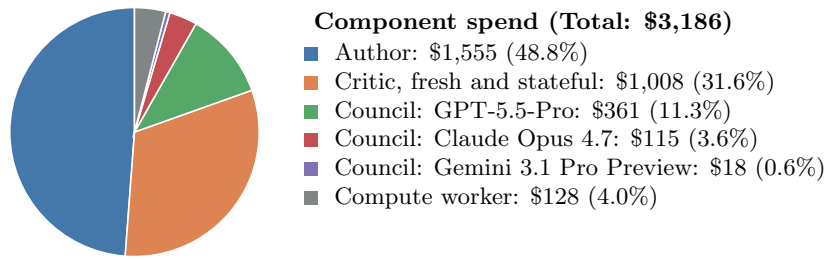
\section{Limitations and Future Work}
\label{sec:discussion}

We briefly discuss several limitations of our work and directions for future work.

\paragraph{Cost}
Running ProofCouncil is expensive, with recorded costs exceeding \$200 per attempted problem. This cost is primarily driven by the repeated use of frontier models for both proof generation and verification, as well as by calls to auxiliary agents such as the LLM council and compute node. Reducing this cost is an important direction for future work. Possible approaches include improving stopping criteria or routing simpler subtasks to cheaper models.

\paragraph{Problem interpretation and evaluation reliability}
Two outputs in our open-question evaluation solved easier interpretations of the stated problems. Since the critic shares the author's reading of the problem statement, such misinterpretations are unlikely to be caught within the system, and detecting them currently requires an expert who knows the intended meaning. More broadly, our results rest on human judgment rather than complete verification: the FirstProof referees graded submissions up to at most minor revisions, three of the 30 researcher problems produced no output because of execution errors, six of the remaining 27 outputs have not yet received feedback, and the progress analysis in \cref{fig:progress-panel-all} is model-generated. In addition, since feedback from earlier runs informed prompt adjustments and bug fixes, the researcher-problem results constitute an adaptive development evaluation rather than an evaluation of a single frozen configuration. The reported numbers should be read with these caveats in mind.

\paragraph{Human readability}
ProofCouncil was optimized primarily for finding correct solutions rather than for producing polished mathematical proofs. Several mathematicians noted that the resulting proofs were dense and would need to be expanded before they could be read by non-experts in the relevant subfield. Improving readability is therefore another important direction for future work. This could involve adding dedicated exposition agents that expand arguments, supply missing background, standardize notation, and rewrite final solutions into a form closer to a research paper.

\section{Conclusion}
\label{sec:conclusion}

In this work, we introduced ProofCouncil, an agentic system for solving open mathematical problems. ProofCouncil solved 6 of the 10 problems in the second batch of the FirstProof challenge up to at most minor revisions, showing the best performance among participating teams. On a separate set of open problems collected from mathematical researchers, ProofCouncil produced five complete solutions and meaningful partial progress in many other cases. To support further work, we release both ProofCouncil and the underlying open-source library for constructing agentic systems as conditional DAGs.
\section*{Use of AI}
We describe how AI was used in this project, including the models and tools involved. The code for ProofCouncil was generated with a combination of Claude Code and Codex, guided by human feedback on what should be generated, how ProofCouncil should be structured, and how the interface and generation process should be executed. GPT-5.5-Pro was used to produce the analysis shown in \cref{fig:progress-panel-all}. Its feedback helped us interpret the runs, identify failure modes, and make the necessary changes. In the paper itself, AI was used only to paraphrase paragraphs for clarity and assist with formatting, while all substantive content was written by humans. In addition, Claude Fable 5 (via Claude Code) produced a detailed referee-style review of the paper before submission, whose suggestions were triaged by the authors, and made some targeted edits following detailed author prompting. One exception is \cref{app:erdos539}, which is a cleaned-up example output from ProofCouncil.
\section*{Acknowledgments}

We thank the organizers of the FirstProof challenge for hosting the competition and providing a platform for evaluating mathematical problem-solving agents. We also thank the researchers at ETH Zurich and other institutions who shared their open problems with us and provided feedback on our solutions, and Jeremy Feusi for helpful discussions and input in the early phase of the project.

This work was supported by an unrestricted gift from Google Ireland Limited to ETH Zurich, supporting work related to ``Research on Evals'', administered through the ETH Zurich Foundation, and by the Swiss National Science Foundation grant 10009122, ``Beyond Benchmark Scores: Analyzing AI Reasoning on Research-Level Mathematics''. J.S. was also supported by SwissMAP. We also received free API credits for the Gemini API, enabling the evaluation of Gemini models; we thank the team at Google for this contribution.

\section*{Contributions}

\begin{itemize}
    \item Johannes Schmitt led the team, designed and implemented the overall system architecture, and served as the primary point of contact for the ETH math researchers and the organizers of the FirstProof challenge.
    \item Tim Gehrunger co-developed the agent-building library, helping Johannes implement the ProofCouncil system and set up the evaluation pipeline.
    \item Jasper Dekoninck helped design the user interface for the library, and served as the main writer of the paper.
    \item Gergely B\'erczi tested the agent, implemented small features in the library, and provided feedback on the paper.
    \item Uri Kreitner tested the agent and provided feedback on the paper.
    \item Liam Price tested the agent and provided feedback on the paper.
    \item David Holmes joined the project after the FirstProof second-batch submission and results, and contributed code to the released library, in particular support for human-in-the-loop workflows (\cref{sec:library}) together with improvements to the run dashboard, resume handling, and cost accounting.
\end{itemize}

\bibliographystyle{colm2026_conference}
\bibliography{references}

\appendix
\crefalias{section}{appendix}
\Crefname{appendix}{Appendix}{Appendices}

\section{Case Study: Erd\H{o}s Problem 539}
\label{sec:erdos539-case-study}
We present a use-case of ProofCouncil's problem-solving capabilities through a detailed account of its partial solution to Erd\H{o}s Problem~539, an open problem in number theory. The model makes meaningful progress on the problem. The proof was verified by human experts, and its main ingredients, including the exponent conclusion, are additionally formalized in Lean by Codex with GPT-5.5; \cref{app:erdos539-lean} describes the precise scope of the formalization.

Erd\H{o}s Problem~539 asks for estimates on the minimum possible size of the cofactor set
\[
Q(A)=\left\{ \frac{a}{\gcdop(a,b)} : a,b\in A\right\}
\]
where \(A\subseteq\Z_{>0}\) has size \(n\).  Equivalently, if
\[
h(n)=\min_{\substack{A\subseteq\Z_{>0}\\ |A|=n}} |Q(A)|,
\]
the problem asks for the growth of \(h(n)\).  The Erd\H{o}s Problems database \citep{erdos539} lists the problem as open; prior to the work reported here, it recorded the bounds
\[
n^{1/2}\ll h(n)\ll n^{2/3}.
\]
The upper bound \(n^{2/3}\) comes from a construction of Freiman--Lev, presented by Granville--Roesler through an equivalent positive-projection formulation; see also Holzman--Lev--Pinchasi for fixed-dimensional lower bounds~\citep{erdos539,granville_roesler_1999,holzman_lev_pinchasi_2008}.

In unpublished joint work, Bollob\'as and Leader previously announced a negative answer to the \(n^{2/3}\) question, that is, that the exponent \(2/3\) above is not optimal; see the paragraph on prior announcements below for the documentary record. We are not aware of a written account of this work and do not know which bound it established, and we therefore make no claim of priority over that announcement. ProofCouncil produced the following improvement on the published bounds. The Erd\H{o}s Problems database has since been updated to record this bound, while the exact asymptotic order of \(h(n)\) remains open.

\begin{theorem}[Erd\H{o}s 539, power exponent]
\label[theorem]{thm:erdos539-maintext}
There is an absolute constant \(C>0\) such that, for every \(n\ge2\),
\[
\frac{1+\sqrt{8n-7}}{2}\le h(n)\le
n^{1/2}\exp(C\sqrt{\log n}).
\]
Consequently,
\[
\lim_{n\to\infty}\frac{\log h(n)}{\log n}=\frac12.
\]
\end{theorem}

The proof, given in \Cref{app:erdos539}, first translates the number-theoretic problem into the positive-orthant difference-set problem
\[
D(F)=(F-F)^+,
\]
then iterates a separated suspension construction starting from the classical two-dimensional antidiagonal strip.  For each fixed suspension depth this gives a fixed-dimensional construction, and optimizing the depth as a function of \(n\) yields the subexponential factor in \Cref{thm:erdos539-maintext}.  Thus the theorem determines the dimension-free power exponent in Erd\H{o}s Problem~539, while leaving the exact order open within the remaining subexponential factor.

\paragraph{Relation to prior announcements.}
The unpublished Bollob\'as--Leader work mentioned above is documented as follows.  Leader's 2009 Warwick abstract for \emph{Positive Projections} explicitly says that the talk included ``a negative answer to the \(n^{2/3}\) question'' and identifies the work as joint with B\'ela Bollob\'as~\citep{leader_positive_2009}.  A 2012 Oxford abstract gives the same Granville--Roesler positive-projection formulation and again refers to joint work with Bollob\'as~\citep{leader_positive_2012}.  We are not aware of any published written account of this work.  \Cref{app:erdos539} records a short self-contained proof of the \(n^{1/2+o(1)}\) upper bound, which strengthens the published \(n^{2/3}\) bound.

\subsection{Proof}
\label{app:erdos539}

This appendix gives a cleaned-up version of the proof of \Cref{thm:erdos539-maintext}, the Erd\H{o}s~539 result found by ProofCouncil.  The accompanying Lean~4 project in \texttt{lean/} formalizes the main ingredients of the proof; \cref{app:erdos539-lean} describes its precise scope.

\subsubsection{Definitions and the positive-projection reformulation}

For \(x=(x_1,\ldots,x_d)\in\Z^d\), put
\[
x^+=(\max(x_1,0),\ldots,\max(x_d,0)).
\]
If \(F\subseteq\Z^d\) is finite, define
\[
D(F)=(F-F)^+=\{(x-y)^+:x,y\in F\}.
\]
For a finite set \(A\subseteq\Z_{>0}\), put
\[
Q(A)=\{a/\gcdop(a,b):a,b\in A\},
\qquad
h(n)=\min_{\substack{A\subseteq\Z_{>0}\\ |A|=n}} |Q(A)|.
\]
It is useful to pass freely between this formulation and the vector formulation
\[
H(n)=\min_{\substack{d\ge1,\ F\subseteq\Z^d\\ |F|=n}} |D(F)|.
\]

\begin{lemma}[Equivalence of formulations]
\label[lemma]{lem:erdos539-equivalence}
For every \(n\ge1\), one has \(h(n)=H(n)\).
\end{lemma}

\begin{proof}
For \(n=1\) both quantities equal \(1\), witnessed by \(A=\{1\}\) and \(F=\{0\}\subseteq\Z\); so assume \(n\ge2\), which ensures that the set \(P\) below is nonempty.

Let \(A\subseteq\Z_{>0}\) be finite with \(|A|=n\), and let \(P\) be the finite set of primes dividing at least one element of \(A\).  The exponent vector \(v(a)=(v_p(a))_{p\in P}\) is injective on \(A\), and unique factorization gives
\[
v(a/\gcdop(a,b))=(v(a)-v(b))^+.
\]
Hence \(|Q(A)|=|D(v(A))|\), so \(H(n)\le h(n)\).

Conversely, let \(F\subseteq\Z^d\) be finite.  Translate \(F\) by some \(t\in\Z^d\) so that \(F+t\subseteq\Z_{\ge0}^d\); this does not change \(D(F)\).  Choose distinct primes \(p_1,\ldots,p_d\), and set
\[
a_x=\prod_i p_i^{x_i+t_i}\qquad (x\in F).
\]
Again by unique factorization, the exponent vectors of the elements of \(Q(\{a_x:x\in F\})\) are exactly the elements of \(D(F)\).  Thus \(h(n)\le H(n)\).
\end{proof}

\subsubsection{The universal square-root lower bound}

\begin{lemma}
\label[lemma]{lem:erdos539-ordinary-diff-lower}
If \(F\subseteq\Z^d\) and \(|F|=n\), then \(|F-F|\ge2n-1\).
\end{lemma}

\begin{proof}
Choose an integer linear map \(L:\Z^d\to\Z\) that is injective on \(F\); for example, outside finitely many proper rational hyperplanes the coefficients of \(L\) work.  Then \(L(F)\) is an \(n\)-element subset \(\{s_1<\cdots<s_n\}\) of \(\Z\), and \(L(F)-L(F)\) contains the distinct elements
\[
s_1-s_n,\ldots,s_n-s_n,s_n-s_{n-1},\ldots,s_n-s_1.
\]
Since \(|L(F-F)|\le |F-F|\), the claim follows.
\end{proof}

\begin{proposition}
\label[proposition]{prop:erdos539-lower}
For every \(n\ge1\),
\[
h(n)\ge \frac{1+\sqrt{8n-7}}2.
\]
\end{proposition}

\begin{proof}
Let \(F\subseteq\Z^d\) have \(|F|=n\), and put \(q=|D(F)|\).  Since
\[
x-y=(x-y)^+-(y-x)^+,
\]
we have \(F-F\subseteq D(F)-D(F)\).  A \(q\)-element set has at most \(q(q-1)+1\) ordinary differences, so \Cref{lem:erdos539-ordinary-diff-lower} gives \(2n-1\le q(q-1)+1\), which is equivalent to the displayed lower bound.  Now use \Cref{lem:erdos539-equivalence}.
\end{proof}

\subsubsection{The base strip}

\begin{lemma}[The base strip]
\label[lemma]{lem:erdos539-strip}
For every integer \(W\ge1\), there is a set \(B_W\subseteq\Z^2\) such that
\[
|B_W|\ge \frac12 W^3,
\qquad
|D(B_W)|\le 3W^2,
\qquad
|B_W-B_W|\le 6W^3.
\]
\end{lemma}

\begin{proof}
Let \(L=W^2\) and
\[
B_W=\{(i,j)\in\Z_{\ge0}^2:0\le i,j\le L,\ L\le i+j\le L+W-1\}.
\]
For each \(0\le t<W\), the line \(i+j=L+t\) contributes \(L-t+1\) points, whence
\[
|B_W|=\sum_{t=0}^{W-1}(L-t+1)\ge W^3/2.
\]

Let \(u=(x-y)^+\) with \(x,y\in B_W\).  If both coordinates of \(u\) are positive, then
\[
u_1+u_2=(x_1+x_2)-(y_1+y_2)\le W-1.
\]
Otherwise \(u\) lies on a coordinate axis and its nonzero coordinate is at most \(L\).  Hence \(D(B_W)\) is contained in the union of \(\{(0,0)\}\), the two axis segments of length \(L\), and the positive integer pairs \((r,s)\) with \(r+s\le W-1\).  Therefore
\[
|D(B_W)|\le2W^2+1+(W-1)(W-2)/2\le3W^2.
\]
Finally, if \(z=x-y\in B_W-B_W\), then \(-W^2\le z_1\le W^2\) and \(|z_1+z_2|\le W-1\), so
\[
|B_W-B_W|\le(2W^2+1)(2W-1)\le6W^3.
\]
\end{proof}

\subsubsection{Separated suspension}

\begin{lemma}[Separated suspension]
\label[lemma]{lem:erdos539-suspension}
Let \(F\subseteq\Z^d\) be nonempty and finite, and write \(m=|F|\), \(q=|D(F)|\), and \(r=|F-F|\).  For every integer \(K\ge1\), there is a set \(\mathcal S_K(F)\subseteq\Z^{2d+1}\) such that
\[
|\mathcal S_K(F)|=Km^2,
\qquad
|D(\mathcal S_K(F))|\le q^2+2(K-1)r,
\qquad
|\mathcal S_K(F)-\mathcal S_K(F)|=(2K-1)r^2.
\]
\end{lemma}

\begin{proof}
Choose an integer
\[
M>\max\{|x_i-x_i'|:x,x'\in F,\ 1\le i\le d\},
\]
and let \(\one=(1,\ldots,1)\in\Z^d\).  Put
\[
\mathcal S_K(F)=\{(j,x+jM\one,y-jM\one):0\le j<K,\ x,y\in F\}.
\]
The size is \(Km^2\).  Compare two points indexed by \((j,x,y)\) and \((\ell,x',y')\).  If \(j=\ell\), their positive difference is
\[
(0,(x-x')^+,(y-y')^+),
\]
giving at most \(q^2\) values.  If \(j>\ell\) and \(a=j-\ell\), then each coordinate of \(x-x'+aM\one\) is strictly positive and each coordinate of \(y-y'-aM\one\) is strictly negative.  Thus the positive difference is
\[
(a,x-x'+aM\one,0),
\]
and for each \(a\) there are at most \(r\) choices.  The case \(j<\ell\) similarly gives
\[
(0,0,y-y'+(\ell-j)M\one),
\]
again at most \(r\) choices for each level gap.  This proves the asserted bound for \(D(\mathcal S_K(F))\).

For ordinary differences, the same pair gives
\[
(j-\ell,x-x'+(j-\ell)M\one,y-y'-(j-\ell)M\one).
\]
The first coordinate determines the level gap, and for each of the \(2K-1\) possible gaps the two ordinary differences \(x-x'\) and \(y-y'\) range independently over \(F-F\).  Hence the ordinary difference set has exactly \((2K-1)r^2\) elements.
\end{proof}

\subsubsection{Iteration and optimization}

Set \(a_s=2^{s+2}-1\) and \(b_s=2^{s+1}\) for \(s\ge0\).

\begin{proposition}[Iterated constructions]
\label[proposition]{prop:erdos539-iteration}
For every \(s\ge0\), there are constants \(c_s,C_s,T_s>0\) and, for each integer \(W\ge1\), a set \(F_s(W)\subseteq\Z^{3\cdot2^s-1}\) satisfying
\[
|F_s(W)|\ge c_sW^{a_s},
\qquad
|D(F_s(W))|\le C_sW^{b_s},
\qquad
|F_s(W)-F_s(W)|\le T_sW^{a_s}.
\]
Moreover \(\log(1/c_s)+\log C_s+\log T_s=O(2^s)\), with an absolute implied constant.
\end{proposition}

\begin{proof}
For \(s=0\), take \(F_0(W)=B_W\), where \(B_W\) is the strip from \Cref{lem:erdos539-strip}; this gives \(a_0=3\), \(b_0=2\), and one may take \(c_0=1/2\), \(C_0=3\), \(T_0=6\).

Assume \(F_s(W)\) has been constructed and define \(F_{s+1}(W)=\mathcal S_W(F_s(W))\).  The dimension changes from \(d\) to \(2d+1\), hence from \(3\cdot2^s-1\) to \(3\cdot2^{s+1}-1\).  \Cref{lem:erdos539-suspension} gives
\[
|F_{s+1}(W)|\ge c_s^2W^{2a_s+1},
\]
\[
|D(F_{s+1}(W))|\le (C_s^2+2T_s)W^{2b_s},
\]
and
\[
|F_{s+1}(W)-F_{s+1}(W)|\le2T_s^2W^{2a_s+1}.
\]
Since \(2a_s+1=a_{s+1}\), \(2b_s=b_{s+1}\), and \(a_s+1=2b_s\), this is the desired induction with
\[
c_{s+1}=c_s^2,
\qquad
C_{s+1}=C_s^2+2T_s,
\qquad
T_{s+1}=2T_s^2.
\]
If \(M_s=\max\{2,1/c_s,C_s,T_s\}\), then \(M_{s+1}\le3M_s^2\), so \(\log M_s=O(2^s)\).
\end{proof}

\begin{proposition}[Fixed-dimensional exponents]
\label[proposition]{prop:erdos539-fixed-dim}
For every integer \(s\ge0\),
\[
h(n)\ll_s n^{\alpha_s},
\qquad
\alpha_s=\frac{b_s}{a_s}=\frac{2^{s+1}}{2^{s+2}-1}
=\frac12+\frac{1}{2(2^{s+2}-1)}.
\]
Equivalently, in dimension \(d_s=3\cdot2^s-1\), the construction gives exponent
\[
\alpha_s=\frac{2(d_s+1)}{4d_s+1}.
\]
\end{proposition}

\begin{proof}
Fix \(s\), and choose \(W=\lceil(n/c_s)^{1/a_s}\rceil\).  Then \(|F_s(W)|\ge n\), and any \(n\)-element subset of \(F_s(W)\) has positive-difference set contained in \(D(F_s(W))\).  By \Cref{lem:erdos539-equivalence,prop:erdos539-iteration},
\[
h(n)\le C_sW^{b_s}\le C_s2^{b_s}c_s^{-b_s/a_s}n^{b_s/a_s}\ll_s n^{\alpha_s}.
\]
The displayed expression in terms of \(d_s\) follows from \(d_s=3\cdot2^s-1\).
\end{proof}

\begin{proof}[Proof of \Cref{thm:erdos539-maintext}]
The lower bound is \Cref{prop:erdos539-lower}.  For the upper bound, the logarithmic control of the constants in \Cref{prop:erdos539-iteration}, together with the proof of \Cref{prop:erdos539-fixed-dim}, gives
\[
h(n)\le \exp(A2^s)n^{1/2+1/(2a_s)}
\]
with an absolute constant \(A\).  For \(n\) large, choose \(s\) so that
\[
2^s\le\sqrt{\log n}<2^{s+1};
\]
the remaining bounded values of \(n\) are absorbed by increasing the final constant.  This optimization lets the ambient dimension \(3\cdot2^s-1\) grow with \(n\), which is allowed in the definition of \(H(n)\).  Since \(a_s\ge2^s\),
\[
\log h(n)
\le \frac12\log n + \frac{\log n}{2a_s}+A2^s
\le \frac12\log n+C\sqrt{\log n}
\]
for an absolute constant \(C\).  The lower bound then gives
\[
\lim_{n\to\infty}\frac{\log h(n)}{\log n}=\frac12.
\]
\end{proof}

\subsection{Lean formalization}
\label{app:erdos539-lean}
Our repository contains a Lean~4 formalization of the main combinatorial and asymptotic ingredients, created by Codex with GPT-5.5 (xhigh). We manually verified its semantic correctness. The project is organized as a Lake package using Lean~4.29.1 and mathlib~4.29.1.  The modules \texttt{Erdos539.NumberBridge} and \texttt{Erdos539.NumberLower} formalize the passage between finite integer sets and positive-difference sets; \texttt{Erdos539.Base}, \texttt{Erdos539.Suspension}, and \texttt{Erdos539.Iteration} formalize the base construction, separated suspension, and iterated upper bounds; and \texttt{Erdos539.Main} collects these ingredients into the logarithmic upper and lower estimates for the dimension-free exponent. Concretely, the development proves the universal lower bound, upper bounds with explicit constants for each fixed suspension depth, and the exponent conclusion \(\lim_{n\to\infty}\log h(n)/\log n=1/2\) of \Cref{thm:erdos539-maintext}. The sharper explicit upper bound \(h(n)\le n^{1/2}\exp(C\sqrt{\log n})\) is at present established only by the informal proof above.


\section{Library Screenshots} \label{app:library}

\Cref{fig:library-execution-start,fig:library-node-add,fig:library-node-view,fig:library-node-edit} show the developer interface used to configure and inspect ProofCouncil workflows. The interface exposes workflow presets as editable DAGs. For more details, we refer to our repository.

\begin{figure*}[t]
    \centering
    \includegraphics[width=\textwidth]{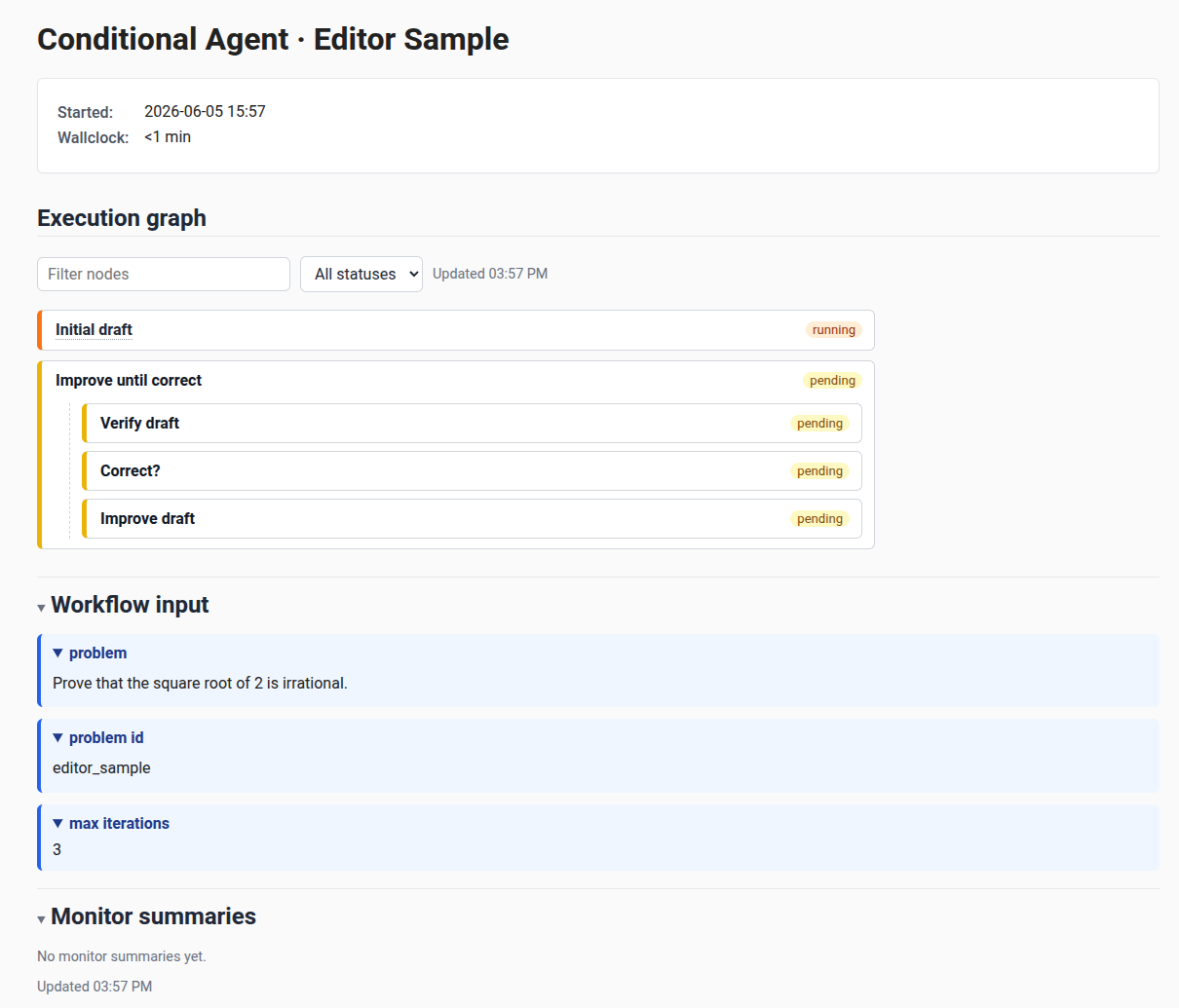}
    \caption{View of the execution page, just after starting a run.}
    \label{fig:library-execution-start}
\end{figure*}

\begin{figure*}[t]
    \centering
    \includegraphics[width=0.55\textwidth]{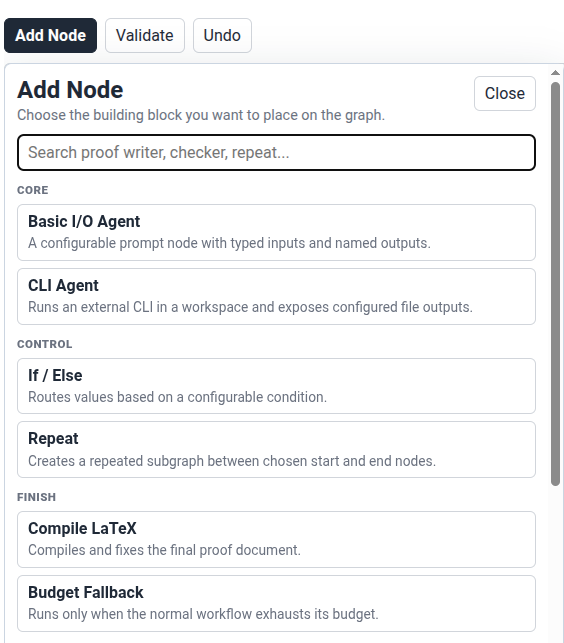}
    \caption{Node-add menu for inserting reusable workflow components into a DAG.}
    \label{fig:library-node-add}
\end{figure*}

\begin{figure*}[t]
    \centering
    \includegraphics[width=\textwidth]{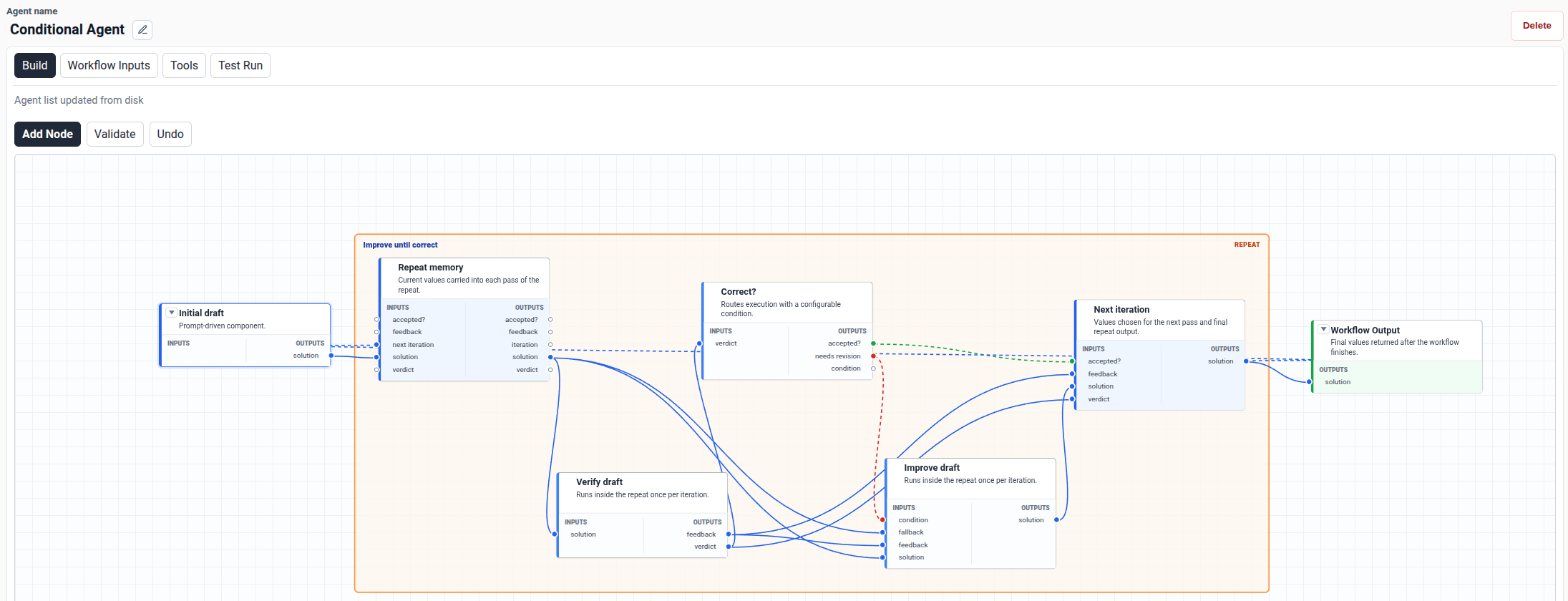}
    \caption{Graph view for inspecting workflow structure and the state of individual nodes.}
    \label{fig:library-node-view}
\end{figure*}

\begin{figure*}[t]
    \centering
    \includegraphics[width=\textwidth]{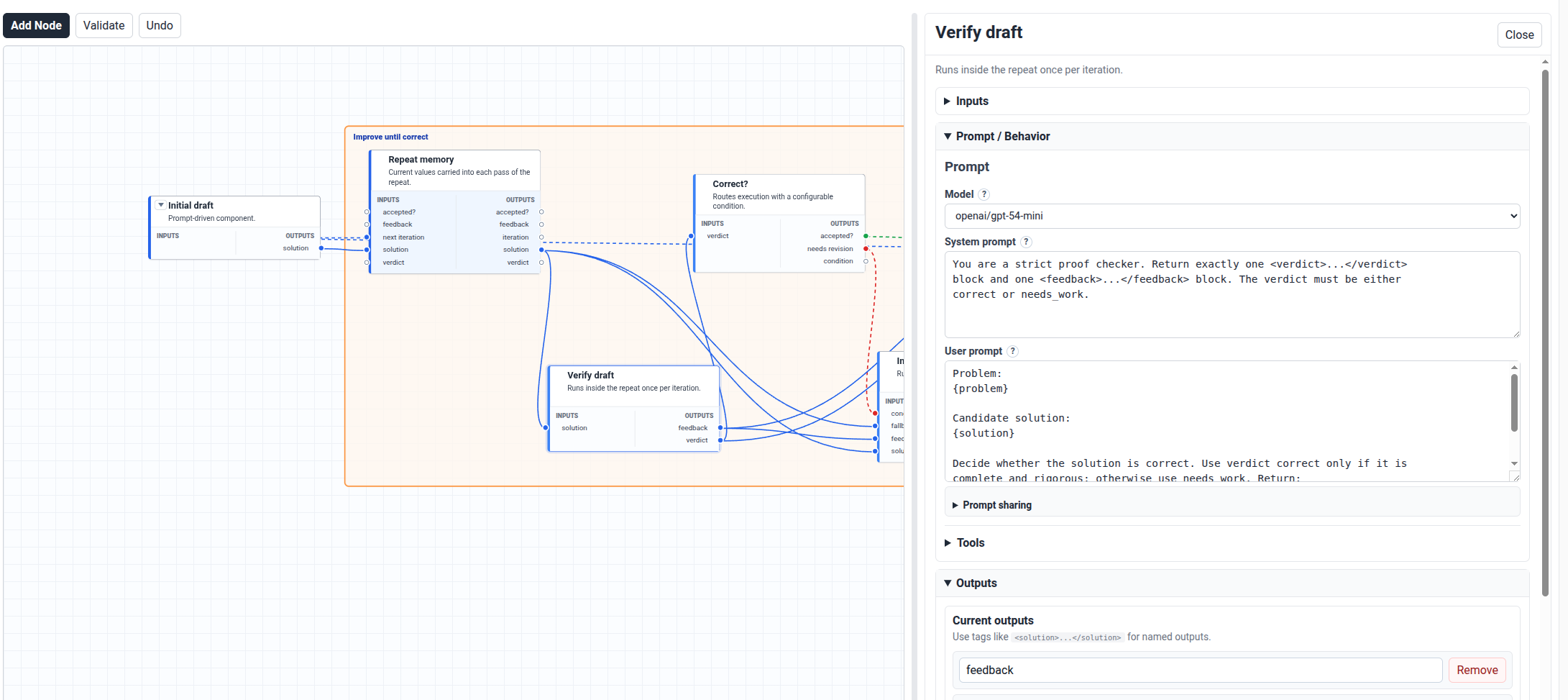}
    \caption{Node editor for changing prompts, model configuration, input bindings, and output schema.}
    \label{fig:library-node-edit}
\end{figure*}

\clearpage
\section{Prompts} \label{app:prompts}

In this section, we provide all prompts used in ProofCouncil.

\newcommand{\renderpromptfile}[2]{%
\begingroup\ttfamily\scriptsize\sloppy\setlength{\parindent}{0pt}\setlength{\parskip}{0.10ex}
\input{prompts/#2}
\par\endgroup
}

\newcommand{\promptfile}[2]{%
\begin{promptbox}{#1}
\begingroup\ttfamily\scriptsize\sloppy\setlength{\parindent}{0pt}\setlength{\parskip}{0.10ex}
\input{prompts/#2}
\par\endgroup

\end{promptbox}
}

\newcommand{\userpromptfile}[2]{%
\begin{userpromptbox}{#1}
\begingroup\ttfamily\scriptsize\sloppy\setlength{\parindent}{0pt}\setlength{\parskip}{0.10ex}
\input{prompts/#2}
\par\endgroup

\end{userpromptbox}
}

\newcommand{\developerpromptfile}[2]{%
\begin{developerpromptbox}{#1}
\begingroup\ttfamily\scriptsize\sloppy\setlength{\parindent}{0pt}\setlength{\parskip}{0.10ex}
\input{prompts/#2}
\par\endgroup

\end{developerpromptbox}
}

\subsection{Author prompts}

\developerpromptfile{Author round 0 developer prompt}{author_round0_system_container.tex}
\userpromptfile{Author round 0 user prompt}{author_round0_user_container.tex}
\developerpromptfile{Author loop developer prompt}{author_loop_system_container.tex}
\userpromptfile{Author loop user prompt}{author_loop_user_container.tex}

\subsection{Critic prompts}
The stateful critic user prompt is appended to the critic's conversation history in every round, but always starts with a fresh critic user prompt at first. When a stateful critic accepts the proof, a fresh critic is called to provide an additional audit. The prompt for this additional fresh critic is slightly different compared to the fresh critic called in the main loop, as it is not given author thinking traces.

\userpromptfile{Fresh critic user prompt}{critic_fresh_user.tex}
\userpromptfile{Fresh Critic user prompt (after stateful acceptance)}{critic_fresh_user_no_thinking.tex}
\userpromptfile{Critic stateful user prompt}{critic_stateful_user.tex}

\subsection{Auxiliary-agent prompts}

\developerpromptfile{Council member developer prompt}{council_member_system.tex}
\userpromptfile{Council member user prompt}{council_member_user.tex}
\promptfile{Compute worker prompt}{compute_worker_prompt.tex}

\subsection{Prescreen prompt}

\developerpromptfile{Prescreen developer prompt}{prescreen_system.tex}
\userpromptfile{Prescreen user prompt}{prescreen_user.tex}

\subsection{Postscreen progress-audit prompt}

\developerpromptfile{Postscreen progress-audit developer prompt}{postscreen_progress_audit_system.tex}
\userpromptfile{Postscreen progress-audit user prompt}{postscreen_progress_audit_user.tex}

\end{document}